\documentclass[sigconf]{acmart}

\AtBeginDocument{%
  \providecommand\BibTeX{{%
    \normalfont B\kern-0.5em{\scshape i\kern-0.25em b}\kern-0.8em\TeX}}}

\copyrightyear{2021} 
\acmYear{2021} 
\setcopyright{acmlicensed}
\acmConference[SIGIR '21]{Proceedings of the 44th International ACM SIGIR Conference on Research and Development in Information Retrieval}{July 11--15, 2021}{Virtual Event, Canada} \acmBooktitle{Proceedings of the 44th International ACM SIGIR Conference on Research and Development in Information Retrieval (SIGIR '21), July 11--15, 2021, Virtual Event, Canada} 
\acmPrice{15.00} 
\acmDOI{10.1145/3404835.3462919} 
\acmISBN{978-1-4503-8037-9/21/07}

\newcommand{\Space}[1]{\mathbb{#1}}
\newcommand{\Set}[1]{\mathcal{#1}}

\newcommand{\ie}{\emph{i.e., }}
\newcommand{\eg}{\emph{e.g., }}

\newcommand{\aka}{\emph{aka. }}

\usepackage{bm}
\usepackage{enumitem}
\usepackage{algorithm}
\usepackage{algorithmic}
\usepackage{multirow}
\usepackage{mathtools}
\usepackage{subfigure}

\begin{document}

\fancyhead{}

\title{AutoDebias: Learning to Debias for Recommendation}

\author{Jiawei Chen$^{1*}$, Hande Dong$^{1*}$, Yang Qiu$^1$, Xiangnan He$^{1\dagger}$, \\Xin Xin$^2$, Liang Chen$^3$, Guli Lin$^4$, Keping Yang$^4$.}

\affiliation{\institution{$^1$University of Science and Technology of China, $^2$University of Glasgow, 
$^3$Sun Yat-Sen University, $^4$Alibaba Group.}} 

\email{{cjwustc@, donghd@mail., yangqiu@mail., hexn@}ustc.edu.cn} 
\email{x.xin.1@research.gla.ac.uk, chenliang6@mail.sysu.edu.cn, {guli.lingl, shaoyao}@taobao.com.} 

\def\authors{Jiawei Chen, Hande Dong, Yang Qiu, Xiangnan He, Xin Xin, Liang Chen, Guli Lin and Keping Yang}

\thanks{$*$ Jiawei Chen and Hande Dong contribute equally to the work.\\$\dagger$ Xiangnan He is the corresponding author.} 

\renewcommand{\shortauthors}{\authors}

\begin{abstract}
Recommender systems rely on user behavior data like ratings and clicks to build personalization model. However, the collected data is observational rather than experimental, causing various biases in the data which significantly affect the learned model. Most existing work for recommendation debiasing, such as the inverse propensity scoring and imputation approaches, focuses on one or two specific biases, lacking the universal capacity that can account for mixed or even unknown biases in the data. 

Towards this research gap, we first analyze the origin of biases from the perspective of \textit{risk discrepancy} that represents the difference between the expectation empirical risk and the true risk. Remarkably, we derive a general learning framework that well summarizes most existing debiasing strategies by specifying some parameters of the general framework. This provides a valuable opportunity to develop a universal solution for debiasing, e.g., by learning the debiasing parameters from data. However, the training data lacks important signal of how the data is biased and what the unbiased data looks like. To move this idea forward, we propose \textit{AotoDebias} that leverages another (small) set of uniform data to optimize the debiasing parameters by solving the bi-level optimization problem with meta-learning. Through theoretical analyses, we derive the generalization bound for AutoDebias and prove its ability to acquire the appropriate debiasing strategy. Extensive experiments on two real datasets and a simulated dataset demonstrated effectiveness of AutoDebias. 
The code is available at \url{https://github.com/DongHande/AutoDebias}. 

\end{abstract}



\maketitle
\section{Introduction}
 \allowdisplaybreaks[4]
\label{sec:introduction}

Being able to provide personalized suggestions to each user, recommender systems (RS) have been widely used in countless online applications.
Recent years have witnessed flourishing publications on recommendation, most of which aim at inventing machine learning model to fit user behavior data~\cite{DBLP:conf/sigir/Yuan0KZ20, DBLP:conf/cikm/SunLWPLOJ19, he2020lightgcn}.
However, these models may be deteriorated in real-world RS, as the behavior data is often full of biases.
In practice, the data is observational rather than experimental, and is often affected by many factors,
including but not limited to self-selection of the user (\textit{selection bias})~\cite{DBLP:conf/uai/MarlinZRS07, DBLP:conf/icml/Hernandez-LobatoHG14a},
exposure mechanism of the system (\textit{exposure bias})~\cite{DBLP:conf/sigir/WangBMN16, DBLP:conf/www/OvaisiAZVZ20},
public opinions (\textit{conformity bias})~\cite{DBLP:conf/recsys/LiuCY16,  krishnan2014methodology} and the display position (\textit{position bias})~\cite{DBLP:journals/sigir/JoachimsGPHG17, DBLP:journals/tois/JoachimsGPHRG07}.
These biases make the data deviate from reflecting user true preference.
Hence, blindly fitting data without considering the data biases would yield unexpected results,
\eg amplifying the long-tail effect~\cite{DBLP:conf/recsys/AbdollahpouriBM17} and previous-model bias~\cite{DBLP:conf/sigir/LiuCDHP020}.

%


Given the wide existence of data biases and their large impact on the learned model, we cannot emphasize too much the importance of properly debiasing for practical RS.
Existing efforts on recommendation (or learning-to-rank) biases can be divided into three major categories:
1) data imputation~\cite{DBLP:conf/icml/Hernandez-LobatoHG14a, DBLP:conf/recsys/Steck13}, which assigns pseudo-labels for missing data to reduce variance,
2) inverse propensity scoring (IPS)~\cite{DBLP:conf/icml/SchnabelSSCJ16, DBLP:conf/sigir/WangBMN16}, a counterfactual technique that reweighs the collected data for an expectation-unbiased learning, and
3) generative modeling~\cite{DBLP:conf/www/LiangCMB16}, which assumes the generation process of data and reduces the biases accordingly.
Despite their effectiveness in some scenarios, we argue that they suffer from two limitations:
\begin{itemize}[leftmargin=*]
\item \textbf{Lacking Universality}. These methods are designed for addressing one or two biases of a specific scenario,
\eg IPS for selection bias~\cite{DBLP:conf/icml/SchnabelSSCJ16}, click model for position bias~\cite{craswell2008experimental}.
Thus, when facing real data that commonly contains multiple types of biases, these methods will fall short.
\item  \textbf{Lacking Adaptivity}. The effectiveness of these methods is guaranteed only when the debiasing configurations (\eg pseudo-labels, propensity scores, or data-generating process) are properly specified.
However, obtaining such proper configurations is quite difficult, requiring domain expertise that thoroughly understands the biases in the data and how they affect the model.
Even worse, the optimal configurations may evolve with time as new users/items/interactions may change the data distribution,
which has been a nightmare for practitioners to manually tune the configurations continually.
\end{itemize}

Considering the shortcomings of existing work, we believe it is essential to develop a universal debiasing solution, which not only accounts for multiple biases and their combinations, but also frees human efforts to identify biases and tune the configurations.
To achieve this goal, we first review the common biases and debiasing strategies, offering two important insights:
(1) Although different biases have different origins and properties,
they all can be formulated as the \textit{risk discrepancy} between the empirical risk and the true risk, resulting from the inconsistency between the distribution for which the training
data is collected and the one used for unbiased test;
(2) The success of most recent debiasing strategies can be attributed to their specific configurations to offset the discrepancy for model training.
Based on the insights, we propose a general debiasing framework by reducing the risk discrepancy,
which subsumes most debiasing strategies --- each strategy can be recovered by specifying the parameters of the framework.
This framework provides a valuable opportunity to develop a universal debiasing solution for recommendation --- we can perform
automatic debiasing by learning the debiasing parameters of the framework.

Now the question lies in how to optimize the debiasing parameters.
Obviously, the biased training data lacks important signals of how the data is biased and what the unbiased data looks like.
To deal with this problem, we propose to leverage another \textit{uniform data} to supervise the learning of debiasing parameter.
The uniform data is assumed to be collected by a random logging policy~\cite{DBLP:conf/icml/SchnabelSSCJ16}, reflecting user preference in an unbiased way.
We make full use of this important evidence, optimizing the debiasing parameters by minimizing the loss on the uniform data.
Specifically, we formulate the process as a bi-level optimization problem, where the debiasing parameters serve as the hyper-parameters for learning the recommender model,
and optimize debiasing parameters by meta-learning technique~\cite{DBLP:conf/icml/FinnAL17}.
We conduct theoretical analyses on the learning framework, proving that:
(1) the optimum learned under such objective is approximate to the best case where biases are properly corrected;
(2) it is able to learn a satisfactory debiasing strategy even if it is trained on a small uniform data.

Lastly, in terms of leveraging uniform data for recommendation, the most relevant work is the recently proposed KDRec~\cite{DBLP:conf/sigir/LiuCDHP020}.
However, we argue that it does not sufficiently exploit the merits of uniform data.
KDRec trains a separate teacher model on the uniform data, and then transfers the model's knowledge to the normal training on biased data. Since uniform data is collected at the expense of degrading user experience, its size is usually rather small. As such, the model trained on it suffers from high variance, decreasing the effectiveness of KDRec.
What's more, it lacks theoretical guarantees, and the inherent mechanism of how teacher model offsets the bias is not entirely understood.
Compared to KDRec, our framework utilizes uniform data in a more theoretically sound way and yields significant empirical improvements.

%

In a nutshell, this work makes the following main contributions:

\begin{itemize}
\item Unifying various biases from the \textit{risk discrepancy} perspective and developing a general debiasing framework that subsumes most debiasing strategies.
\item Proposing a new method that leverages uniform data to learn optimal debiasing strategy with theoretical guarantees.
\item Conducting experiments on three types of data (explicit and implicit feedback, and simulated data of list feedback) to validate the effectiveness of our proposal.
\end{itemize}

\section{Preliminary}

%

In this section, we formulate the recommendation task and review various biases in it from \textit{risk discrepancy} perspective.


\subsection{Task Formulation}

\begin{table}[t!]
    \centering
    \setlength{\abovecaptionskip}{0.2cm}
 \setlength{\belowcaptionskip}{-0.00cm}
    \caption{Notations and Definitions.}
    \resizebox{.48\textwidth}{!}{%
    \begin{tabular}{|c|c|ccc|}
    \hline
    Notations           &   Annotations     \\
    \hline
    $D_T$& \begin{tabular}[c]{@{}c@{}}a biased training set with entries $\{(u_k,i_k,r_k)\}_{1\le k \le |D_T|}$ \\ collected from user interaction history\end{tabular} \\ \hline
    $D_U$ & \begin{tabular}[c]{@{}c@{}}a unbiased uniform set with entries $\{(u_l,i_l,r_l)\}_{1\le l \le |D_U|}$ \\ collected with uniform logging policy \end{tabular} \\ \hline
    $p_T(u,i,r)$& the data distribution for which $D_T$ collected  \\ \hline
    $p_U(u,i,r)$& the ideal unbiased data distribution\\ \hline
    $f_{\theta}(.,.)$ & RS model that maps a user-item pair into the prediction \\ \hline
    $\hat L_T(f)$& the empirical risk of $f$ on $D_T$\\ \hline
    $L(f)$& the true risk of $f$\\ \hline
    $S_1$&   \begin{tabular}[c]{@{}c@{}} the subspace of $\Set U\times \Set I \times \Set R$ with constraints: \\ ${p_T}(u,i,k) > 0, {p_U}(u,i,k) > 0$  \end{tabular}\\ \hline
    $S_0$& \begin{tabular}[c]{@{}c@{}} the subspace of $\Set U\times \Set I \times \Set R$ with constraints:  \\ ${p_T}(u,i,k) = 0, {p_U}(u,i,k) > 0$\end{tabular} \\
    \hline
    \end{tabular}%
    }
     \vspace{-0.3cm}
    \label{notation}
\end{table}

Suppose we have a recommender system with a user set $\Set U$ and an item set $\Set I$. Let $u$ (or $i$) denotes a user (or an item) in $\Set U$ (or $\Set I$).
Let $r \in \Set R$ denotes the feedback (\eg rating values, clicks, and retention time) given by a user to an item.
The collected history behavior data $D_T$ can be notated as a set of triplets $\{(u_k,i_k,r_k)\}_{1\le k \le |D_T|}$ generated from an unknown distribution $p_T(u,i,r)$ over user-item-label space $\Set U\times \Set I \times \Set R$. The task of a recommendation system can be stated as follows: learning a recommendation model from $D_T$ so that it can capture user preference and make a high-quality recommendation. Formally, let $\delta (.,.)$ denotes the error function between the prediction and the ground truth label.
The goal of recommendation is to learn a parametric function $f_{\theta}: \Set U\times \Set I \to \Set R$ from the available dataset $D_T$ to minimize the following \textit{true risk}:
\begin{equation}
    \label{risk_equ}
    \begin{split}
        L(f)=\mathbb E_{p_U(u,i,r)}[\delta (f(u,i),r)], \\
    \end{split}
\end{equation}
where $f_{\theta}$ can be implemented by a specific recommendation model with a learnable parameters $\theta$. We remark that in this paper we may simply omit subscript '$\theta$' in the notations for clear presentation. $p_U(u,i,r)$ denotes the ideal unbiased data distribution for model testing. This distribution can be factorized as the product of the user-item pair distribution $p_U(u,i)$ (often supposed as uniform) and the factual preference distribution for each user-item pair $p_U(r|u,i)$.

Since the true risk is not accessible, the learning is conducted on the training set $D_T$ by optimizing the following \textit{empirical risk}:
\begin{equation}
    \label{em_risk_equ}
    \begin{split}
        \hat L_T(f)=\frac{1}{|D_T|}\sum_{k=1}^{|D_T|}\delta (f(u_k,i_k),r_{k}). \\
    \end{split}
\end{equation}
If the empirical risk $\hat L_T(f)$ is an unbiased estimator of the true risk $L(f)$, \ie $\mathbb E_{p_T}[L_T(f)]=L(f)$,
the PAC learning theory~\cite{haussler1990probably} states that the learned model will be approximately optimal if we have sufficiently large training data.

\subsection{Biases in Recommendation}
\label{Statistical perspective}
However, as various biases occur in real-world data collection,
the training data distribution $p_T$ is often inconsistent with the ideal unbiased distribution $p_U$.
Training data only gives a skewed snapshot of user preference, making the recommendation model sink into sub-optimal result.
Figure \ref{fg:gap} illustrates this phenomenon.
The red curve denotes the true risk function, while the blue curve denotes the expectation of the biased empirical risk function $\mathbb E_{p_T}[L_T(f)]$.
As the two risks are expected over different distribution, they will behave rather differently even in their optimum (\ie $f^*$ versus $f^T$).
It means that even if a sufficiently large training set is provided and the model arrives at empirical optimal point $f^T$,
there still exists a gap $\Delta L$ between the optimum $L(f^*)$ and the empirical one $L(f^T)$.
Above analysis reveals the impact of bias on recommendation: bias will incur the discrepancy between the true risk and the expected empirical risk.
Blindly fitting a recommendation model without considering the risk discrepancy will result in inferior performance.

\begin{figure}[t!]
\centering
\includegraphics[width=0.35\textwidth]{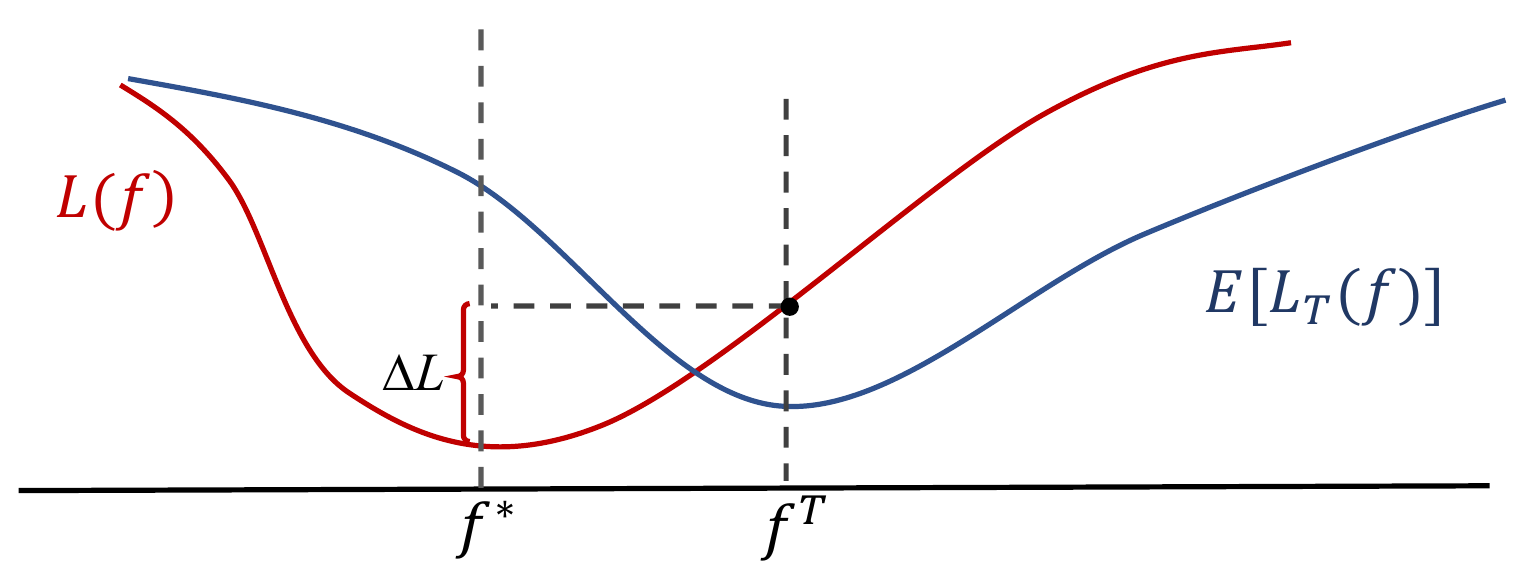}
\setlength{\abovecaptionskip}{0.2cm}
 \setlength{\belowcaptionskip}{-0.00cm}
 \caption{An example of the biased empirical risk.}
  \vspace{-0.3cm}
\label{fg:gap}
\end{figure}

Recent work has identified various biases in recommendation.
In this subsection, we will review these biases from risk discrepancy perspective to help the readers to better understand their properties and negative effect.
We refer to~\cite{DBLP:journals/corr/abs-2010-03240} and categorize data biases into the following four classes:


\textbf{Selection bias} \textit{happens as users are free to choose which items to rate,
so that the observed ratings are not a representative sample of all ratings}~\cite{DBLP:journals/corr/abs-2010-03240}.
Selection bias can be easily understood from the \textit{risk discrepancy} perspective --- it skews the user-item pair distribution $p_T(u,i)$ from the ideal uniform one $p_U(u,i)$. Typically, $p_T(u,i)$ inclines to the pairs with high rating values. Learning a recommendation model under such biased data will easily overestimate the preferences of users to items.


\textbf{Conformity bias} \textit{happens as users tend to behave similarly to the others in a group,
even if doing so goes against their own judgment}~\cite{DBLP:journals/corr/abs-2010-03240}.
Conformity bias distorts label distribution $p_T(r|u,i)$ in conformity to the public opinions, making the feedback do not always signify user true preference, \ie $p_T(r|u,i) \ne p_U(r|u,i)$.

\textbf{Exposure bias} \textit{happens in implicit feedback data as users are only exposed to a part of specific items}~\cite{DBLP:journals/corr/abs-2010-03240}.
It would be difficult to understand exposure bias from the above definition, but straightforward from risk discrepancy perspective.
On the one hand, users generate behaviors on exposed items, making the observed user-item distribution $p_T(u,i)$ deviate from the ideal one $p_U(u,i)$.
On the other hand, implicit feedback data only has positive feedback observed $p_T(r=1|u,i)=1$.
Such positive-only data will cause ambiguity in the interpretation of unobserved interactions --- they may caused by non-exposure or by dislike.



\textbf{Position bias} \textit{happens as users tend to interact with items in higher position of the recommendation list}~\cite{DBLP:journals/corr/abs-2010-03240}.
Under position bias, training data distribution $p_T(u,i,r)$ will be sensitive to the item display position and fail to reflect user preference faithfully.
On the one hand, the ranking position will affect the chance of the item exposure to the user~\cite{DBLP:conf/www/OvaisiAZVZ20}, \ie $p_T(u,i) \ne p_U(u,i)$.
On the other hand, as users often trust the recommendation system, their judgments also will be affected by the position, \ie $p_T(r|u,i) \ne p_U(r|u,i)$.

Generally speaking, above biases can be summarized as a type of risk discrepancy  --- they cause training data distribution $p_T(u,i,r)$ deviate from the ideal unbiased one $p_U(u,i,r)$. This insight motivates us to develop a powerful framework that directly conquers the risk discrepancy, enabling the elimination of the mixture of above biases or even unknown biases in the data.

\section{A general debiasing framework}
In this section, we provide a general debiasing framework that can account for various kinds of biases in recommendation data. We then discuss how it subsumes most existing debiasing strategies.

\subsection{Debiasing Empirical Risk}
\begin{figure}[t!]
    \centering
    \includegraphics[width=0.49\textwidth]{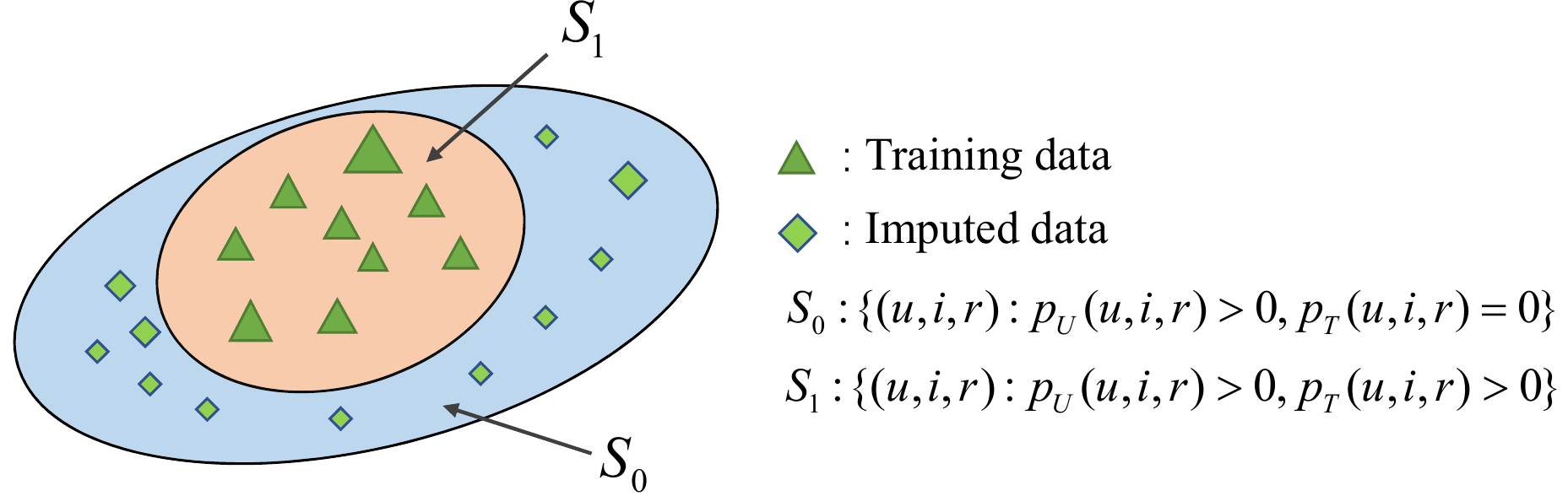}
    \setlength{\abovecaptionskip}{0.2cm}
 \setlength{\belowcaptionskip}{-0.00cm}
    \caption{An illustration of the fact that training data distribution only covers part of ideal data distribution. We need impute pseudo-data (marked as diamond) to the blank region $S_0$, where the mark size reflects their weights. }
     \vspace{-0.3cm}
    \label{fg:shiyi}
\end{figure}

The analyses presented in Section~\ref{Statistical perspective} show that various biases account for the risk discrepancy.
For unbiased learning, we need to re-design the empirical risk function so that its expectation under biased training distribution is consistent with the true risk.
By comparing $L(f)$ with $\mathbb E_{p_T}[\hat L_T(f)]$,
we can see that the discrepancy origins from the data distribution, where the effect of each data has been skewed in empirical risk.
To deal with it, we can re-weigh training data and obtain a reweighted empirical risk function:
\begin{equation}
    \begin{split}
{{\hat L}_T}(f|{w^{(1)}}) = \frac{1}{|D_T|}\sum\limits_{k=1}^{|D_T|} {w_{k}^{(1)}L(f({u_k},{i_k}),{r_k})}.
  \end{split}
\end{equation}
When the parameter $w^{(1)}$ is properly specified, \ie $w_k^{(1)} = \frac{{p_U({u_k},{i_k},{r_k})}}{{{p_T}({u_k},{i_k},{r_k})}}$,
the skewness of each sampled training data is corrected. Such empirical risk ${{\hat L}_T}(f|{w^{(1)}})$ seems be an unbiased estimation of the true risk as:
\begin{equation}
    \begin{split}
{E_{{p_T}}}[{{\hat L}_T}(f|w^{(1)})] = \sum\limits_{(u,i,r) \in S_1 } {p_U(u,i,r)\delta (f(u,i),r)},
  \end{split}
\end{equation}
where $S_1$ denotes the user-item-label subspace satisfying ${p_T}(u,i,r) > 0, {p_U}(u,i,r) > 0$, \ie $S_1\equiv \{(u,i,r) \in \Set U \times \Set I \times \Set R:{p_T}(u,i,r) > 0,{p_U}(u,i,r) > 0\}$.
However, the equivalence of ${\mathbb E_{{p_T}}}[{{\hat L}_T(f|w^{(1)})}]$ with $L(f)$ is actually not held.
As illustrated in Figure~\ref{fg:shiyi}, the training data distribution $p_T$ only covers a part of regions (\ie $S_1$) of the support of $p_U$,
while does not have probability in other regions ($S_0\equiv \{(u,i,r) \in U \times I \times R:{p_T}(u,i,k)=0, {p_U}(u,i,k) > 0\}$).
It means that even if a sufficiently large training data is available,
it still provides partial user-item pairs only, while leaves the rest blank.
Learning on $S_1$ only will suffer especially when $S_0$ and $S_1$ exhibit different patterns.
This situation is common in practice: the system tends to display popular items~\cite{DBLP:conf/recsys/KamishimaAAS14},
while a considerable number of unpopular items have little chance to be exposed.
To deal with this problem, we need to impute pseudo-data to the blank region. In fact, the equivalent transformation of $L(f)$ is:
\begin{equation}
    \begin{split}
&L(f)= \smashoperator{\sum\limits_{u\in \Set U,i\in \Set I,r \in \Set R}} {p_U(u,i,r)\delta (f(u,i),r)} \\
 &=\smashoperator{\sum\limits_{(u,i,r)\in S_1 }} {p_U(u,i,r)\delta (f(u,i),r)} + \smashoperator{\sum\limits_{ (u,i,r)\in S_0 }} {p_U(u,i,r)\delta (f(u,i),r)} \\
 &= {E_{{p_T}}}[\frac{1}{|D_T|}\smashoperator{\sum\limits_{k=1}^{|D_T|}} {w_k^{(1)}\delta (f({u_k},{i_k}),{r_k})}] + \smashoperator{\sum\limits_{u \in \Set U,i \in \Set I}} {w_{ui}^{(2)}\delta (f(u,i),{m_{ui}})}.
 \end{split}
 \label{eq:lf}
\end{equation}
The last equation holds when the parameters $\phi \equiv \{{w^{(1)}},{w^{(2)}},m\}$ are set as:
 \begin{equation}
    \begin{split}
w_k^{(1)} &= \frac{{p_U({u_k},{i_k},{r_k})}}{{{p_T}({u_k},{i_k},{r_k})}}\\
w_{ui}^{(2)} &= \sum\limits_{r\in \Set R} {p_U(u,i,r)\mathbf I[{p_T}(u,i,r) = 0]} \\
{m_{ui}} &= {E_{p_U(r|u,i)}}[r\mathbf I[{p_T}(u,i,r) = 0]],
 \end{split}
 \label{eq:para}
\end{equation}
where $\mathbf I[.]$ denotes indicator function.
Here we absorb the expectation over $\delta (.,.)$ into the pseudo-label, \ie $\mathbb E[\delta (.,.)]=\delta (.,\mathbb E[.])$.
It holds for many commonly-used loss functions, such as L2, L1 and cross-entropy.
Thus, we define the following empirical risk function as:
 \begin{equation}
    \begin{split}
{{\hat L}_T}(f|\phi) = \frac{1}{|D_T|}\sum\limits_{k=1}^{|D_T|} {w_k^{(1)}\delta (f({u_k},{i_k}),{r_k})}  + \smashoperator{\sum\limits_{u \in \Set U,i \in \Set I}} {w_{ui}^{(2)}\delta (f(u,i),{m_{ui}})},
\end{split}
\label{training_loss}
\end{equation}
which is an unbiased estimator of the true risk when the parameters are properly specified.
We remark that there may exist multiple solutions of $\phi$ but at least one (\ie Equation (\ref{eq:para})) that makes equation $\hat L_T(f|\phi)=L(f)$ hold.
This generic empirical risk function provides a valuable opportunity to develop a universal solution --- achieving automatic debiasing by learning the debiasing parameters $\phi$.

\subsection{Link to Related Work}
To show the university of our framework, we review representative debiasing strategies and discuss how the framework subsumes them.

\vspace{+5pt}
\noindent\textbf{Selection bias.} Existing methods are mainly three types:

(1) Inverse Propensity Score (IPS)~\cite{DBLP:conf/icml/SchnabelSSCJ16} reweighs the collected data for an unbiased learning,
defining the empirical risk as:
\begin{equation}
     \mathcal{\hat L}_{IPS}(f) = \frac{1}{|\Set U|| \Set I|}\sum\limits_{k=1}^{|D_T|} \frac{1}{q_{u_ki_k}} \delta(f(u_k,i_k),r_k),
    \label{IPS_selection_bias}
\end{equation}
where $q_{u_ki_k}$ is defined as propensity, which estimates the probability of the data to be observed.
The framework recovers it by setting $w^{(1)}_k=\frac{|D_T|}{q_{u_ki_k}|\Set U|| \Set I|}$, $w^{(2)}_k=0$.

 (2) Data Imputation~\cite{DBLP:conf/icml/Hernandez-LobatoHG14a} assigns pseudo-labels for missing data and optimizes the following risk function:
\begin{equation}
 {{ \hat L}_{IM}}(f) = \frac{1}{|\Set U| |\Set I|}(\sum\limits_{k=1}^{|D_T|} {\delta (f({u_k},{i_k}),{r_k})}  + \smashoperator{\sum\limits_{u \in \Set U,i \in \Set I}} \lambda \delta (f(u,i),m_{ui})),
\end{equation}
where $m_{ui}$ denotes the imputed labels which can be specified heuristically~\cite{DBLP:conf/icdm/HuKV08} or inferred by a dedicated model~\cite{DBLP:conf/recsys/MarlinZ09, saito2020asymmetric,DBLP:conf/icdm/Chen0ESFC18}.
$\lambda$ is to control the contribution of the imputed data. It is a special case of our framework if we set $w^{(1)}_k=\frac{|D_T|}{|\Set U| |\Set I|}, w^{(2)}_{ui}=\frac{\lambda}{|\Set U| |\Set I|}$.

(3) Doubly Robust~\cite{wang2019doubly} combines the above models for more robustness --- the capability to remain unbiased if either the imputed data or propensities are accurate. It optimizes:
\begin{equation}
{ { \hat L}_{DR}}(f) = \frac{1}{{|\Set U||\Set I|}}\sum\limits_{u \in |\Set U|,i \in |\Set I|} {\left( {\delta (f(u,i),m_{ui}) + \frac{{{O_{ui}}{d_{ui}}}}{{{q_{ui}}}}} \right)},
\end{equation}
where ${d_{ui}} = \delta (f(u,i),r_{ui}^o) - \delta (f(u,i),r_{ui}^i)$ denotes the difference between the predicted error and imputed error.
$r_{ui}^o$ denotes the observed rating values. $O_{ui}$ denotes whether the interaction of $(u,i)$ is observed.
Our framework can recover this one by setting $w^{(1)}_k=\frac{|D_T|}{q_{u_ki_k}|\Set U||\Set I|}$, $w^{(2)}_{ui}=\frac{1}{|\Set U||\Set I|}-\frac{O_{ui}}{q_{ui}|\Set U||\Set I|}$.

\vspace{+5pt}
\noindent\textbf{Conformity bias.} The conformity effect can be offset by optimizing~\cite{DBLP:conf/recsys/LiuCY16, DBLP:conf/sigir/MaKL09}:
\begin{equation}
{\hat L_{OFF}}(f) = \frac{1}{|D_T|}\sum\limits_{k=1}^{|D_T|} {{{(\alpha {r_k} + (1 - \alpha ){b_{{u_k}{i_k}}} - f({u_k},{i_k}))}^2}},
\end{equation}
where ${b_{{u_k}{i_k}}}$ denotes the introduced bias term,
which can be specified as the average rating over all users\cite{DBLP:conf/recsys/LiuCY16} or social friends~\cite{DBLP:conf/sigir/MaKL09}.
$\alpha$ controls the effect of conformity. Our framework subsumes it by setting $w^{(1)}_k = 1, w^{(2)}_{ui} = O_{ui}(1-\lambda), m_{ui} =  - {b_{ui}}$.

\vspace{+5pt}
\noindent\textbf{Exposure bias.} Existing methods are mainly two types:

(1) Negative Weighting, treats unobserved interactions as negative and down-weighs their contributions~\cite{DBLP:conf/icdm/HuKV08}:
\begin{equation}
{{\hat L}_{NW}}(f) =  \frac{1}{|D_T|}\sum\limits_{k=1}^{|D_T|} {\delta (f({u_k},{i_k}),{r_k})}  + \smashoperator{\sum\limits_{(u,i)\in \Set U\times \Set I:O_{ui}=0}} {{a_{ui}}\delta (f(u,i),0)},
\end{equation}
where parameter $a_{ui}$ indicates how likely the item is exposed to a user,
which can be specified heuristically~\cite{DBLP:conf/icdm/HuKV08} or by an exposure model~\cite{DBLP:conf/www/LiangCMB16,chen2018modeling,chen2019samwalker}.
We can recover this method by setting $w^{(1)}_k=1$, $w^{(2)}_{ui}=a_{ui}(1-O_{ui})$, $m_{ui}=0$.

(2) IPS Variant, reweighs observed data and imputes missing data~\cite{saito2020unbiased}:
\begin{equation}
{\hat L_{IPSV}}(f) = \smashoperator{\sum\limits_{k=1}^{|D_T|}} {\frac{1}{{{{q}_{u_ki_k}}}}\delta (f({u_k},{i_k}),{r_k})}  + \smashoperator{\sum\limits_{u\in \Set U,i \in \Set I}} {(1 - \frac{{{O_{ui}}}}{{{{q}_{ui}}}})\delta (f(u,i),0)}.
\end{equation}
We can recover it by setting $w^{(1)}_k=\frac{|D_T|}{\hat q_{u_ki_k}|\Set U||\Set I|}$, $w^{(2)}_{ui}=\frac{1}{|\Set U||\Set I|}-\frac{O_{ui}}{q_{ui}|\Set U||\Set I|}$, $m_{ui}=0$.
It is worth mentioning that the unbiasedness of Variant IPS conditions on that the training distribution $p_T$ covers the whole support of $p_U$,
however it seldom holds in practice.

\vspace{+5pt}
\noindent\textbf{Position bias.} The most popular strategy is IPS~\cite{DBLP:conf/sigir/AiBLGC18}, which reweighs the collected data with a position-aware score ${\hat q}_{t_k}$:
\begin{equation}
{\hat L_{IPS}}(f) = \frac{1}{|D_T|}\sum\limits_{k=1}^{|D_T|} {\frac{{\delta (f({u_k},{i_k}),{r_k})}}{{{{ q}_{t_k}}}}}.
\end{equation}
It can be recovered by setting $w^{(1)}_k=\frac{1}{{{{ q}_{t_k}}}}$, $w^{(2)}_{ui}=0$.

\section{AutoDebias Algorithm}
We now consider how to optimize the aforementioned framework.
Since the training data lacks important signals of how the data is biased and what the unbiased data looks like, it is impossible to learn proper debiasing parameters from such data.
To deal with this problem, another \textit{uniform data} $D_U$ is introduced to supervise the learning of debiasing parameters. The uniform data contains a list of triplets $\{(u_l,i_l,r_l)\}_{1\le l \le |D_U|}$, which is assumed to be collected by a random logging policy, providing a gold standard evidence on the unbiased recommendation performance. We make full use of this evidence and optimize $\phi$ towards better performance on uniform data. Specifically, the learning process can be formulated as a \textit{meta learning} process with:

\textit{Base learner}: The base recommendation model is optimized on the training data with current debiasing parameters $\phi$:
\begin{equation}
{\theta ^*(\phi)} = \mathop {\arg \min }\limits_\theta  {{\hat L}_T}({f_\theta }|\phi ) \label{eq:base}.
\end{equation}
where debiasing parameters $\phi$ can be seen as the hyper-parameters of the base learner.

\textit{Meta leaner}: Given the learned base recommendation model $\theta^*(\phi)$ from training data with the hyper-parameters $\phi$, $\phi$ is optimized towards better recommendation performance on uniform data:
\begin{equation}
{\phi ^*} = \mathop {\arg \min }\limits_\phi \frac{1}{{|{D_U}|}}\sum\limits_{l = 1}^{|{D_U}|} {\delta ({f_{{\theta ^*(\phi)}}}({u_l},{i_l}),{r_l})}. \label{eq:meta}
\end{equation}
For better description, the empirical risk on uniform data is marked as ${{\hat L}_U}({f_\theta })$, \ie $ {{\hat L}_U}({f_\theta })= \frac{1}{{|{D_U}|}}\sum_{l = 1}^{|{D_U}|} {\delta ({f_\theta }({u_l},{i_l}),{r_l})}$.

As the uniform data is often collected in a small scale, directly learning all parameters in $\phi$ from it will incur over-fitting. To deal with this problem, $\phi$ can be re-parameterized with a concise \textit{meta model}. This treatment can reduce the number of parameters and encode useful information (e.g., user id, observed feedback) into debiasing. In this paper, we simply choose a linear model for implementation as:
\begin{equation}
    \begin{split}
w_k^{(1)} &= \exp (\varphi _1^T[{\mathbf x_{{u_k}}} \circ {\mathbf x_{{i_k}}} \circ {\mathbf e_{r_k}}])\\
w_{ui}^{(2)} &= \exp (\varphi _2^T[{\mathbf x_u} \circ {\mathbf x_i} \circ {\mathbf e_{O_{ui}}}])\\
{m_{ui}} &= \sigma (\varphi _3^T[{\mathbf e_{r_{ui}}} \circ {\mathbf e_{O_{ui}}}]),
    \end{split}
    \label{eq:li}
\end{equation}
where $\mathbf{x}_u$ \text{and} $\mathbf{x}_i$ denote the feature vectors (\eg one-hot vector of its id) of user $u$ and item $i$, respectively; $\mathbf e_r$, $\mathbf e_{O_{ui}}$ are one-hot vectors of $r$ and $O_{ui}$; the mark $\circ$ denotes the concatenation operation; $\varphi  \equiv \{ {\varphi _1},{\varphi _2},{\varphi _3}\}$ are surrogate parameters to be learned; $\sigma(.)$ denotes the activated function controlling the scale of the imputation values, \eg $tanh(\cdot)$.
One might concern that modeling $\phi$ with a meta model potentially induce inductive bias, restricting the flexibility of $\phi$ and making it fail to arrive at the global optimum.
In fact, our framework is relatively robust to such inductive bias, which has been validated in Section 5.

\textbf{Model learning.}
Note that obtaining optimal $\phi^*$ involves nested loops of optimization --- updating $\phi$ a step forward requires a loop of full training of $\theta$, which is expensive.
To deal with this problem, we consider to update $\theta$ and $\phi$ alternately in a loop.
That is, for each training iteration, we make a tentative updating of recommendation model with current $\phi$ and inspect its performance on the uniform data. The loss on uniform data will give feedback signal to update meta model. To be exact, as illustrated in Figure \ref{Meta_model}, we perform the following training procedure in each iteration:
\begin{itemize}[leftmargin=*]
\item \textbf{Assumed update of $\theta$.} As the black arrows in Figure 3, we make an assumed updating of $\theta$:
\begin{equation}
    \begin{split}
        \theta'(\phi) = \theta - \eta_1  {\nabla _\theta }{{\hat L}_T}({f_\theta }|\phi ),
    \end{split}
\end{equation}
where we update $\theta$ using gradient descent with learning rate $\eta_1$ .

\item \textbf{Update of $\phi$ ($\varphi$).} As the blue arrows shown in figure 3, we test $\theta'(\phi)$ on the uniform data with $\hat L_U$. The loss function gives a feedback signal (gradient) to update the meta model $\varphi$:
    \begin{equation}
    \begin{split}
{\varphi} &\leftarrow {\varphi - \eta_2 {\nabla _\varphi } {{\hat L}_U}({f_{\theta'(\phi)} }) }.
      \end{split}
\end{equation}
 The gradient can be calculated by using the back-propagation along the chain
     ${{\hat L}_U}({f_{\theta'(\phi)} }) \to \theta'(\phi) \to {\nabla _\theta }{{\hat L}_T}({f_\theta }|\phi )) \to \phi \to \varphi$.

\item \textbf{Update of $\theta$.} Given the updated $\phi$, we update $\theta$ actually:
\begin{equation}
    \begin{split}
        \theta \leftarrow \theta - \eta_1  {\nabla _\theta }{{\hat L}_T}({f_\theta }|\phi )).
    \end{split}
\end{equation}
\end{itemize}

\begin{figure}[t!]
    \centering
    \includegraphics[width=0.35\textwidth]{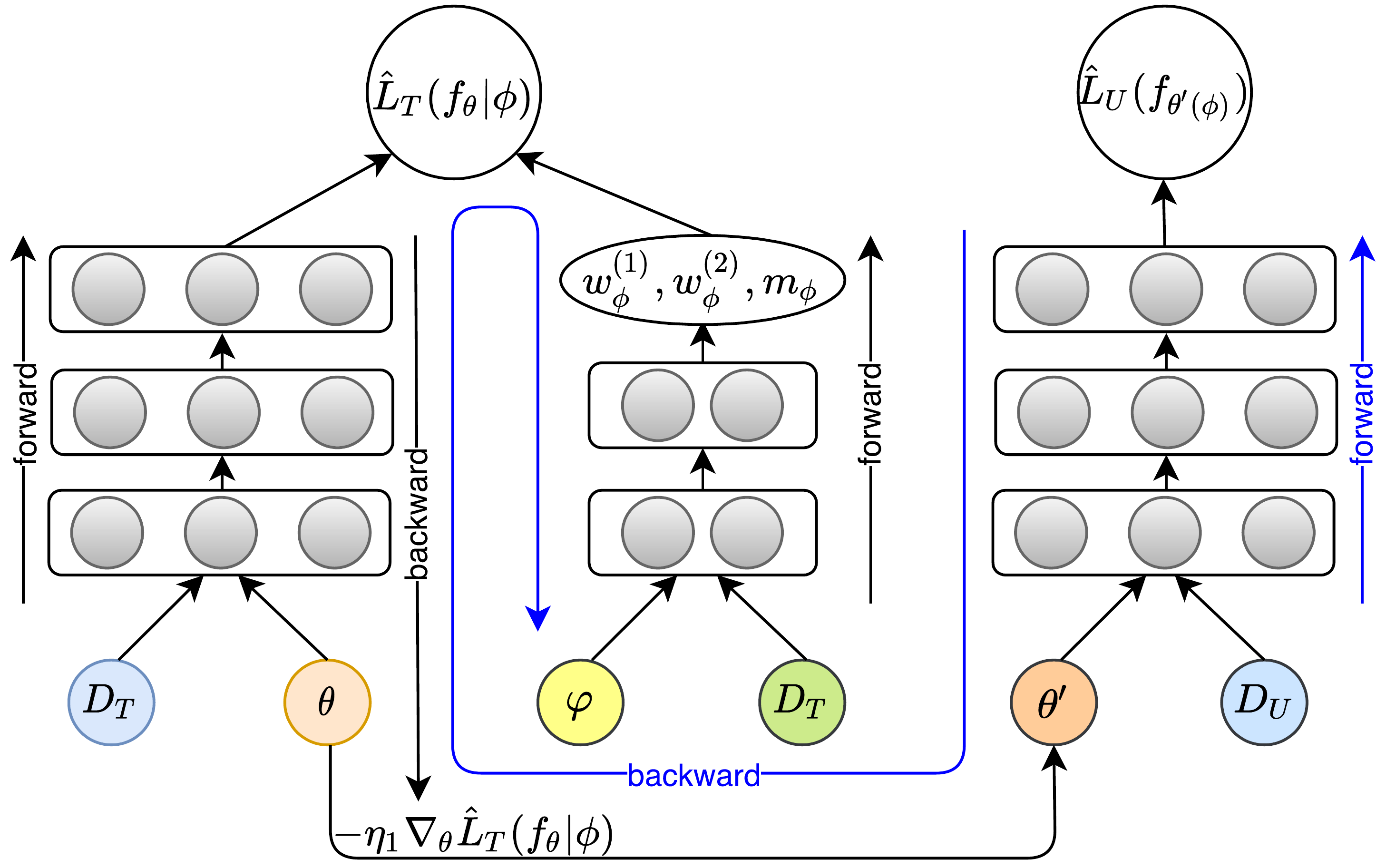}
    \caption{The working flow of AutoDebias, consists of three steps:  (1) tentatively updating $\theta$ to $\theta'$ on the training data $D_T$ with current $\phi$ (black arrows); (2) updating $\phi$ based on $\theta'$ on the uniform data (blue arrows); (3) actually updating $\theta$ with the updated $\phi$ (black arrows).}
    \vspace{-0.3cm}
    \label{Meta_model}
\end{figure}

While this alternative optimization strategy is not guaranteed to find the global optimum, it empirically works well for bi-level optimization problem~\cite{DBLP:conf/icml/FinnAL17}.

\section{Theoretical analysis}

In this section, we conduct theoretical analyses to answer the following questions:
(1) Can AutoDebias acquire the optimal $\phi$ making $E_{p_{T}}[\hat L_T(f)]$ consistent with $L(f)$?
(2) How does inductive bias in meta model affect recommendation performance?

\textbf{Addressing question (1).} We first give some notations.
Let $\Phi$ and $\Space F$ denote the hypothesis space of the meta model and the recommendation model;
$f^*$ denotes the optimum function that minimizes $L(f)$;
$\hat f(\phi)$ denotes empirical optimum under specific $\phi$ by optimizing \ie  $\hat f(\phi ) = \arg {\min _{f \in \Space F}}{{\hat L}_T}(f|\phi )$;
$\phi^*$, the optimal debiasing parameters making $E_{p_T}[L_T(f|\phi^*)]=E[L(f)]$ hold, is supposed to be contained in $\Phi$ in this question, then the general case will be discussed in next question.
${\phi ^o}$ denotes surrogate optimum from $\Phi$ that minimizes ${\phi ^o}= {\arg \min }_{\phi \in \Phi} L(\hat f(\phi ))$;
$\hat \phi$ denotes the empirical $\phi$ from the empirical risk on uniform data $L_U(\hat f(\phi ))$.
As the empirical risk $\hat L_T(f|\phi^*)$ is an unbiased estimation of the true risk, we have:
\newtheorem{lem}{Lemma}
\begin{lem}
\label{la1}
\textbf{Generalization bound of $L(\hat f({\phi ^*}) )$.} For any finite hypothesis space of recommendation models $\Space F=\{f_1, f_2, \cdots , f_{|\Space F|}\}$, with probability of at least $1-\eta$, we have the following generalization error bound of the learned model $\hat f$:
\begin{align}
L(\hat f({\phi ^*})) \le L({f^*}) + \sqrt {\frac{{{S_{{w^{(1)}}}}}\rho^2}{{2|{D_T}|}}\log \frac{{2|\Space F|}}{\eta }},
\end{align}
where ${{S_{{w^{(1)}}}}}$ denotes mean-square of $w^{(1)}$ with ${S_{{w^{(1)}}}} = \sum\limits_{k = 1}^{|{D_T}|} {{{(w_k^{(1)})}^2}}$. $\rho$ denotes the bound of loss function $\delta$.
\end{lem}
We omit the lemma proof due to the limit space. In fact, the proof is a variant of the proof of Corollary 4.6 in \cite{shalev2014understanding}, where we use our debiasing empirical risk and a more strict Hoeffding inequality. For convenient, $\sqrt {\frac{{{S_{{w^{(1)}}}}}\rho^2}{{2|{D_T}|}}\log \frac{{2|\Space F|}}{\eta }}$ is marked as $\varepsilon_1$.

Note that when the space $\Phi$ contains optimum $\phi^*$, the relations $L({f^*}) \le L(\hat f({\phi ^o})) \le L(\hat f({\phi ^*})) \le L({f^*})+\varepsilon_1$ hold, we have:
\newtheorem{cor}{Corollary}
\begin{cor}
\label{co1}
When the meta hypothesis space $\Phi$ contains optimum $\phi^*$, with probability of at least $1-\delta$, the differences between $L({f^*})$, $L(\hat f({\phi ^o}))$, $L(\hat f({\phi ^*}))$ are bounded with:
\begin{align}
|L({f^*}) - L(\hat f({\phi ^o}))| < {\varepsilon _1},|L(\hat f({\phi ^o})) - L(\hat f({\phi ^*}))| < {\varepsilon _1}.
\label{eq:cor}
\end{align}
\end{cor}
This corollary suggests that when the training data is sufficiently large, the surrogate parameter $\phi ^o$ is approximately correct, making the recommendation model $f$ trained in an  unbiased manner and finally arrive at the approximate optimum.

Note that $L_U(\hat f(\phi))$ is an unbiased estimation of $L(\hat f(\phi))$, similar to lemma 1, we have:
\begin{lem}
\label{la2}
For any finite hypothesis space of meta model $\Phi=\{\phi_1, \phi_2, \cdots , \phi_{|\Phi|}\}$, with probability of at least $1-\eta$, we have:
\begin{align}
L(\hat f({\hat \phi})) \le L({\hat f(\phi^o)}) + \sqrt {\frac{\rho^2}{{2|{D_U}|}}\log \frac{{2|\Phi|}}{\eta }}.
\end{align}
\end{lem}
This lemma details the relation of $\hat \phi$ and $\phi^o$. By combining lemma 2 and corollary 1, we have the following generalization bound of the framework:

\begin{lem}
\label{la3}
\textbf{Generalization error bound of AutoDeBias.} For any recommendation model $f$ and meta model $\phi$ with finite hypothesis space $\Space F$ and $\Phi$,  when $\Phi$ contains $\phi^*$, with probability of at least $1-\eta$,  the generalization error of the model that is trained with AutoDebias, is bounded by:
\begin{align}
L(\hat f({\hat \phi})) < L({f^*}){\rm{ + }}\sqrt {\frac{{{S_{{w^{(1)}}}}\rho^2 }}{{2|{D_T}|}}\log \frac{{4|\Space F|}}{\eta }} {\rm{ + }}\sqrt {\frac{{{\rho ^2}}}{{2|{D_U}|}}\log \frac{{4|\Phi |}}{\eta }}.
\end{align}
Also, $L(\hat f({\hat \phi}))$ does not deviate from $L(\hat f({\phi ^*}))$ more than:
\begin{align}
|L(\hat f({\hat \phi})) - L(\hat f({\phi ^*}))|<\sqrt {\frac{{{S_{{w^{(1)}}}}\rho^2 }}{{2|{D_T}|}}\log \frac{{4|\Space F|}}{\eta }} {\rm{ + }}\sqrt {\frac{{{\rho ^2}}}{{2|{D_U}|}}\log \frac{{4|\Phi |}}{\eta }}.
\end{align}
\end{lem}
This lemma proves that when sufficient training data and uniform data are provided, AutoDeBias can acquire an almost optimal debiasing strategy, making the recommendation model arrive at the approximate optimum. 

\textbf{Addressing question (2).} Inductive bias occurs in a small-capacity meta model,
making the hypothesis space $\Phi$ do not contain $\phi^*$ and Equation (\ref{eq:cor}) do not hold accordingly. Nevertheless, we have:

\begin{lem}
\label{la4}
\textbf{Generalization error bound of AutoDeBias with inductive bias.}  For any recommendation model $f$ and meta model $\phi$ with finite hypothesis space $\Space F$ and $\Phi$, the generalization error of the model that is trained with AutoDebias, is bounded by:
\begin{equation}
    \begin{split}
L(\hat f(\hat \phi )) &< L({f^*}){\rm{ + }}\sqrt {\frac{{S_{{w^{(1)}}}^o\rho }}{{2|{D_T}|}}\log \frac{{4|\Space F|}}{\eta }} {\rm{ + }}\sqrt {\frac{{{\rho ^2}}}{{2|{D_U}|}}\log \frac{{4|\Phi |}}{\eta }} \\
 &+ ({\Delta _{(1)}} + {\Delta _{(2)}})\rho  + {\Delta _m}{\rho _d},
    \end{split}
\end{equation}
where ${S_{{w^{(1)}}}^o}$ denotes the mean-square of the surrogate ${w^o}^{(1)}$ in $\phi^o$ with ${S^o_{{w^{(1)}}}} = \sum\limits_{k = 1}^{|{D_T}|} {{({{w^o}^{(1)}_k})^2}}$;
and ${\Delta _{(1)}}$, ${\Delta _{(2)}}$, ${\Delta _{m}}$ denotes the expected deviation between ${\phi ^o}$ and ${\phi ^*}$ caused by inductive bias in terms of $w^{(1)}$, $w^{(2)}$, $m$ respectively,
\ie ${\Delta _{(1)}} = \sum\limits_{k = 1}^{|{D_T}|} {{p_U}({u_k},{i_k},{r_k})|\frac{{{w^*}_k^{(1)} - {w^o}_k^{(1)}}}{{{w^*}_k^{(1)}}}|}$,
${\Delta _{(2)}} = \sum\limits_{u \in U,i \in I} {|{w^*}_{ui}^{(2)} - {w^o}_{ui}^{(2)}|}$,
${\Delta _m} = \sum\limits_{u \in U,i \in I} {{w^*}_{ui}^{(2)}|m_{ui}^* - m_{ui}^o|}$; ${\rho _d}$ denotes the bound of the gradient ${\nabla}\delta (.,.)$.
\end{lem}
\begin{proof}
Due to the limited space, here we just give brief proof.
For convenience, we mark $({\Delta _{(1)}} + {\Delta _{(2)}}\rho ) + {\Delta_m}{\rho _d}$ as $\varepsilon_3$.
Let $L_s(f|\phi)$ denote the surrogate risk with parameter $\phi$, that is, ${L_s}(f|\phi ) = {E_{{p_T}}}[{{\hat L}_T}(f|\phi)]$.
Naturally, ${L_s}(f|\phi)$ deviates from the true risk $L(f)$.
But their difference, which can be calculated from Equation (\ref{eq:lf}),
is bounded by ${\varepsilon _3}$. Let $f^*(\phi)$ denote the optimum of the surrogate risk ${L_s}(f|\phi)$ with the parameters $\phi$.
The inequality ${L_s}(f^*(\phi^o)|\phi^o)<L(f^*)+{\varepsilon _3}$ holds.
Further, note that $L_T(f|\phi^o)$ is an unbiased estimation of $L_s(f|\phi^o)$.
The difference between $L(\hat f({\phi ^o}))$ and ${L_s}(\hat f({\phi ^o}))$ is bounded.
Combining these conclusions with the lemma 3, lemma 4 is proofed.
\end{proof}
Lemma 4 tells us that although the inductive bias occurs, the generalization error is bounded.
In fact, constraining the capacity of the meta model with inductive bias, often can improve model performance.
On the one hand, the size of hypothesis space $\Phi$ is reduced;
On the other hand, the surrogate optimum $\phi^o$ in the constraint space $\Phi$ often has smaller variance,
reducing the mean-square ${S_{{w^{(1)}}}^o}$ and further tightening the bound.

\section{Experiments}

In this section, we conduct experiments to evaluate the performance of our proposed AutoDebias.
Our experiments are intended to address the following research questions:
\begin{itemize}
    \item [\textbf{RQ1:}] Does AutoDebias outperform SOTA debiasing methods?
    \item [\textbf{RQ2:}] How do different components (\ie learning of $w^{(1)}$, $w^{(2)}$ or $m$) affect AutoDebias performance?
    \item [\textbf{RQ3:}] How do learned debaising parameters $\phi$ correct data bias?
    \item [\textbf{RQ4:}] Is AutoDebias universal to handle various biases?
\end{itemize}


\subsection{Experimental Setup}
\label{experiment setup}

In order to validate the universality of AutoDebias, our experiments are on
three types of data: explicit feedback, implicit feedback, and the feedback on recommendation list.

\textbf{Dataset.} For explicit feedback, we refer to ~\cite{DBLP:conf/sigir/LiuCDHP020} and use two public datasets (\textit{Yahoo!R3} and \textit{Coat}) for our experiments.
Both datasets contain a set of biased data collecting the normal interactions of users in the platform,
and a small set of unbiased data from stochastic experiment where items are assigned randomly.
Following~\cite{DBLP:conf/sigir/LiuCDHP020}, we regard the biased data as training set $D_T$, while split the unbiased data into three parts:
5\% for uniform set $D_U$ to help training, 5\% for validation set $D_{V}$ to tune the hyper-parameters, and 90\% for test set $D_{Te}$ to evaluate the model.
The ratings are binarized with threshold 3.
That is, the observed rating value larger than 3 is labeled as positive ($r=1$), otherwise negative $r=-1$.

To generate implicit feedback data, we also use the aforementioned datasets, but abandon the negative feedback of the training data. this treatment is the same as the recent work ~\cite{saito2020unbiased}.

We also adopt a synthetic dataset \textit{Simulation}, which simulates user feedback on recommendation lists
to validate that AutoDebias can handle the situation when both position bias and selection bias occur.
Here we use synthetic dataset as there is no public dataset that contains ground-truth of unbiased data for model evaluation. The simulation process consists of four steps: (1) we follow ~\cite{DBLP:conf/www/SunKNS19} to generated the ground-truth preference scores $r_{ui}$ of 500 users on 500 items. For each user, we randomly split their feedback into two parts: 150 for unbiased data set ($D_1$) and 350 remaining for simulating biased data ($D_2$). (2) For each user, 25 items are sampled randomly from $D_2$ to train a recommendation model $M$. $M$ returns recommendation lists containing top-25 items for each user, which induces selection bias in data. (3) We then refer to ~\cite{DBLP:conf/kdd/GleichL11} to simulate user feedback on recommendation lists, where position bias has been considered. Specifically, the click is generated with the probability of $\mathop{min}(\frac{{r}_{u,i} }{2*p^{1/2}}, 1)$.
The statistics of those datasets is in Table~\ref{Dataset}.
\begin{table}[t]
    \centering
    \setlength{\abovecaptionskip}{0.2cm}
 \setlength{\belowcaptionskip}{-0.00cm}
    \caption{Statistics of the datasets.}
    \resizebox{.45\textwidth}{!}{%
        \begin{tabular}{ccccccc}
        \hline
        Dataset     & Users  & Items  & Training & Uniform & Validation  & Test \\
        \hline
        Yahoo!R3   & 15,400  & 1,000  & 311,704 & 2,700 & 2,700  & 48,600       \\
        Coat       & 290  & 300  & 6,960 & 232 & 232  & 4,176   \\
        Simulation  & 500  & 500  & 12,500 & 3,750 & 3,750 & 67,500       \\
        \hline
        \end{tabular}
    }
     \vspace{-0.3cm}
    \label{Dataset}
\end{table}

\textbf{Evaluation metrics.} We adopt the following metrics to evaluate recommendation performance:
\begin{itemize}
    \item \textbf{NLL} evaluates the performance of the predictions with:
    \begin{equation}
        NLL = - \frac{1}{|D_{te}|} \sum_{(u,i,r)\in D_{te}} log \left(1+e^{-r*f_\theta(u,i)}\right),
    \end{equation}
    \item \textbf{AUC} evaluates the performance of rankings: 
    \begin{equation}
        AUC = \frac{\sum_{(u,i)\in D_{te}^+}\hat Z_{u,i} - (|D_{te}^+|+1)(|D_{te}^+|)/2}{(|D_{te}^+|) * (|D_{te}| - |D_{te}^+|)},
    \end{equation}
    where $|D_{te}^+|$ detotes the number of postive data in $D_{te}$,
    $\hat Z_{u,i}$ denotes the rank position of a positive feedback $(u,i)$.

    \item \textbf{NDCG@k} measures the quality of recommendation through discounted importance based on position:
    \begin{equation}
     \begin{split}
        DC{G_u}@k = \sum\limits_{(u,i) \in {D_{te}}} {\frac{{\mathbf I({{\hat Z}_{u,i}} \le k)}}{{log({{\hat Z}_{u,i}} + 1)}}} \\
        NDCG@k = \frac{1}{|\Set{U}|} \sum_{u\in \Set{U}} \frac{DCG_u@k}{IDCG_u@k},
         \end{split}
    \end{equation}
    where $IDCG_u@k$ is the ideal $DCG_u@k$.
\end{itemize}

\textbf{Implementation details.}
Matrix Factorization (MF) has been selected as a benchmark recommendation model for experiments,
and it would be straightforward to replace it with more sophisticated models such as Factorization Machine~\cite{rendle2012factorization}, or Neural Network~\cite{DBLP:conf/www/HeLZNHC17}.
SGD has been adopted for optimizing base model and Adam for meta model. 
Grid search is used to find the best hyper-parameters based on the performance on validate set.
We optimize the base model with SGD and meta model with Adam. The search space of learning rate and weight decay are [1e-4, 1e-3,
1e-2, 1e-1].

\subsection{Performance Comparison on Explicit Feedback (RQ1)}
\textbf{Baseline.} (1) MF(biased), MF(uniform) and MF(combine):
the basic matrix factorization model that trained on ${D}_T$, ${D}_U$, and $D_T + D_U$, respectively;
(2) Inverse propensity score (IPS)~\cite{DBLP:conf/icml/SchnabelSSCJ16}: a counterfactual technique that reweighs the
collected data. We follow~\cite{DBLP:conf/icml/SchnabelSSCJ16} and calculate the propensity with naive bayes;
(3) Doubly robust (DR)~\cite{wang2019doubly}: combining data imputation and inverse propensity score;
(4) KD-Label~\cite{DBLP:conf/sigir/LiuCDHP020}: the state-of-the-art method that transfers the unbiased information with a teacher model.
We refer to ~\cite{DBLP:conf/sigir/LiuCDHP020} and choose the best label-based distillation for comparison.
(5) CausE~\cite{DBLP:conf/sigir/LiuCDHP020}: distilling unbiased information with an extra alignment term.

\textbf{Performance comparison.}
Table~\ref{result of explicit} presents the recommendation performance of the compared methods in terms of three evaluation metrics.
The boldface font denotes the winner in that column. We have the following observations:
(1) Overall, AutoDebias outperforms all compared methods on all datasets for all metrics.
Especially in the dataset Yahoo!R3, the improvements are rather impressive --- 5.6\% and 11.2\% in terms of NLL and NDCG@5.
This result validates that our proposed AutoDebias can learn better debiasing configurations than compared methods.
(2) MF(uniform) that is directly trained on uniform data performs quite terribly. The reason is that uniform data is often in a quite small scale. Training a recommendation model on uniform data will suffer from serious over-fitting. This phenomenon also affect the performance of KD-based methods (KD-label and CausE). Thus, they perform worse than AutoDebias with a certain margin.

\begin{table}[t]
    \setlength{\abovecaptionskip}{0.2cm}
 \setlength{\belowcaptionskip}{-0.00cm}
\caption{Performance comparisons on explicit feedback data. The boldface font denotes the winner in that column.}
\label{result of explicit}
\resizebox{.45\textwidth}{!}{%
\begin{tabular}{c|ccc|ccc}
\hline
\multirow{2}{*}{Method} & \multicolumn{3}{c}{Yahoo!R3}  & \multicolumn{3}{c}{Coat} \\ \cline{2-7}
                        & NLL    & AUC    &  NDCG@5   & NLL   &   AUC  & NDCG@5 \\ \hline
MF(biased)              &    -0.587   &      0.727       &  0.550    &  -0.539  &  0.747    &  0.500  \\
MF(uniform)              &    -0.513    &    0.573   &   0.449     & -0.623    &  0.580    &  0.358   \\
MF(combine)             &   -0.580   &    0.730    &  0.554  &  -0.538     &     0.750    &  0.504    \\
IPS             &  -0.448  &    0.723   &   0.549  &    -0.515   &    0.759   &   0.509   \\
DR                 &   -0.444  &    0.723  &  0.552  &   \textbf{-0.512}    &  0.765    & 0.521    \\
CausE        &    -0.579    &   0.731    &  0.551  & -0.529  &   0.762     &  0.500  \\
KD-Label            &  -0.632  &    0.740   &  0.580   & -0.593  &   0.748   &  0.504   \\
AutoDebias        &  \textbf{-0.419}  &    \textbf{0.741}   &   \textbf{0.645}   &  \textbf{-0.512}  &   \textbf{0.766}   &  \textbf{0.522}   \\ \hline
\end{tabular}
}
 \vspace{-0.2cm}
\end{table}

To further validate the high efficiency of AutoDebias on utilizing uniform data, we test these methods with varying size of uniform data as shown in figure ~\ref{Robustness}.
We can find AutoDebias consistently outperforms compared methods. More impressively, AutoDebias still perform well even if a quite small scale of uniform data is provided (\eg Ratio=1\%), while KD-label falls short in this situation --- performs close to or even worse than MF(biased).

\begin{figure}[t!]
    \centering
    \setlength{\abovecaptionskip}{0.2cm}
 \setlength{\belowcaptionskip}{-0.00cm}
    \includegraphics[width=0.43\textwidth]{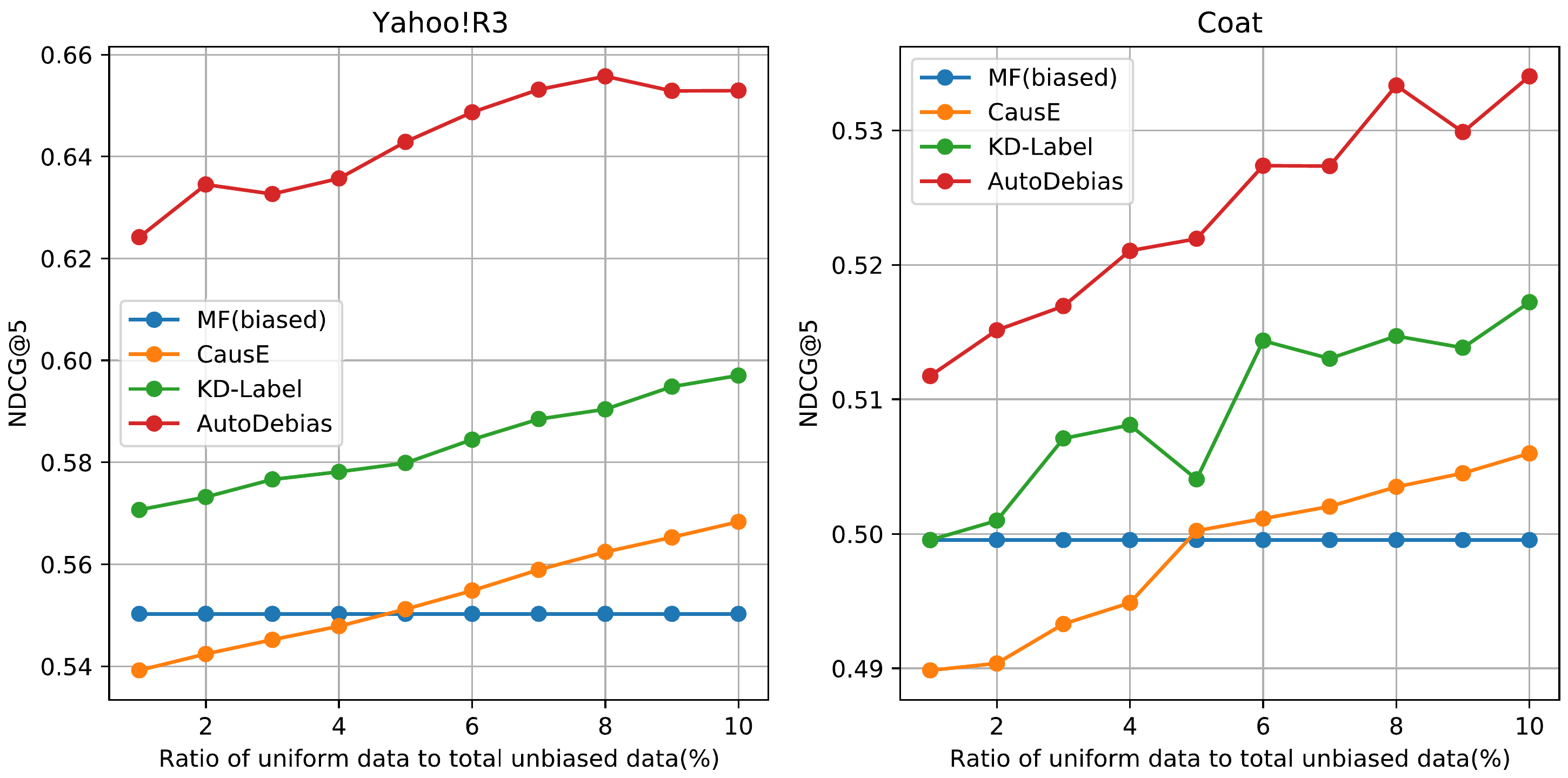}
    \caption{Performance comparisons with varying uniform data size. Here we sample different percentages of total unbiased data as uniform data.  }
    \label{Robustness}
     \vspace{-0.2cm}
\end{figure}

\subsection{Ablation Study (RQ2)}
We conduct ablation study to explore whether the three components $w^{(1)}$, $w^{(2)}$ and $m$ are all desired to be introduced.
Table~\ref{Ablation} shows the performance of the ablated model where different components are removed.
From this table, we can conclude that all the parameters are important for debiasing. Introducing $w^{(1)}$ and $m$ consistently boost the performance. Despite introducing $w^{(2)}$ harms the performance a bit on Coat, it brings an impressive improvement on Yahoo!R3.
It is worth emphasising that even if imputation modular is removed, AutoDebias-w1 still outperforms IPS. This results validates the effectiveness of the meta-learning algorithm on finding the proper $w^{(1)}$.

\subsection{Exploratory Analysis (RQ3)}
To answer question RQ3, we now explore the learned meta parameters from the dataset Yahoo!R3. This exploration provides insights into how AutoDebias corrects data bias.

\textbf{Exploring imputation values $m$:} Table~\ref{imputation} presents the imputation values for user-item pairs with different kinds of feedback. We can find the imputation value for the positive user-item pairs is larger than missing pairs, while the value for missing pairs is large than negative pair. More interestingly, we can find the optimal imputation value for missing pair is rather small (\eg -0.913). This phenomenon is consistent with the finding in \cite{DBLP:conf/recsys/MarlinZ09} that a user dislikes a random-selected item in a high probability.

\begin{table}[t!]
\small
    \setlength{\abovecaptionskip}{0.2cm}
 \setlength{\belowcaptionskip}{-0.00cm}
    \caption{Ablation study in terms of NDCG@5. }
    \label{Ablation}
    \begin{tabular}{ccccccc}
    \hline
    Method&   $w^{(1)}$ & m & $w^{(2)}$ & Yahoo!R3 & Coat \\ \hline
    MF  & $\times$  & $\times$ & $\times$  & 0.550  & 0.500 \\
    AutoDebias-w1  & \checkmark & $\times$ & $\times$  & 0.573  & 0.510 \\
    AutoDebias-w1m  & \checkmark & \checkmark & $\times$  & 0.581  & 0.526 \\
    AutoDebias   & \checkmark & \checkmark & \checkmark  & 0.645  & 0.521 \\ \hline
    \end{tabular}
     \vspace{-0.2cm}

\end{table}

\begin{table}
    \centering
          \setlength{\abovecaptionskip}{0.2cm}
 \setlength{\belowcaptionskip}{-0.00cm}
      \caption{The learned parameters $m$ for different feedback.}
      \label{imputation}
        \begin{tabular}{ccc}
        \hline
        Positive ($r=1$) & Missing & Negative ($r=-1$)  \\ \hline
        -0.992  & -0.913 & 0.227   \\\hline
        \end{tabular}
         \vspace{-0.2cm}
\end{table}

\begin{table}
    \centering
       \setlength{\abovecaptionskip}{0.2cm}
 \setlength{\belowcaptionskip}{-0.00cm}
        \caption{The learned parameters $w_r^{(1)}$ in different models.}
        
        \label{IPW_rating}
        \begin{tabular}{ccc}
        \hline
        Method & $r=-1$  & $r=1$  \\ \hline
        AutoDebias-w1  & 3.37  & 1.56 \\
        AutoDebias& 0.29  & 4.77  \\\hline
        \end{tabular}
         \vspace{-0.2cm}
\end{table}


\textbf{Exploring weights $w^{(1)}$:} Note that $w^{(1)}$ is implemented with a linear model while the feature vector consists of three one hot vectors in terms of user id, item id and the feedback. The linear parameter $\varphi_1$ can be divided into three parts ${\varphi _1}{\rm{ = [}}{\varphi _{11}}^\circ {\varphi _{12}}^\circ {\varphi _{13}}{\rm{]}}$ and $w^{(1)}$ can be decomposed as a product of user-dependent weights $w^{(1)}_u$, the item-dependent weights $w^{(1)}_i$ and the label-dependent weights $w^{(1)}_r$, \ie $ w_k^{(1)} = \exp (\varphi _{11}^T{{\bf{x}}_{{u_k}}})\exp (\varphi _{12}^T{{\bf{x}}_{{i_k}}})\exp (\varphi _{13}^T{{\bf{x}}_{{r_k}}}) = w_{{u_k}}^{(1)}w_{{i_k}}^{(1)}w_{{r_k}}^{(1)}$. Here we explore the distribution of the learned $w^{(1)}_u, w^{(1)}_i, w^{(1)}_r$ to provide the insights.

Table~\ref{IPW_rating} shows the value of learned $w^{(1)}_r$ for ablated AutoDebias-w1 and AutoDebias. We can find that $w^{(1)}_{r}$ for $r=-1$ is larger than $w^{(1)}_{r}$ for $r=1$ in AutoDebias-w1, while $w^{(1)}_{r}$ for $r=-1$ is smaller than $w^{(1)}_{r}$ for $r=1$ in AutoDebias. This interesting phenomenon can be explained as following. AutoDebias-w1 that does not use data imputation, will up-weight the negative instances to account for such selection bias. For AutoDebias, as it has given negative imputed labels for missing data, it does not need to address such under-represented problem and turn to learn other aspects such as data confidence. This result validates the adaptation of our AutoDebias.

We also explore the distribution of $w^{(1)}_i$ with the item popularity as presented in Figure ~\ref{weight}.
We can find: (1) The weights decreases as their popularity increases.
This phenomenon is consistent with our intuition that the popular items are likely to be exposed (\aka popularity bias) and their feedback is over-represented and need down-weighing to mitigate bias.
(2) The variance of the weights of the popular items is larger than unpopular.
 The reason is that some popular items are high-quality, on which the feedback is indeed valuable, while the feedback of others may be caused by the conformity effect and is unreliable. As a result, the weights exhibit more diverse.

To further validate the superiority of our model on addressing such popularity bias. We draw the plots of the performance of different methods in terms of popular and unpopular items in Figure ~\ref{weight}. The popular items are defined as the top 20\% items according to their positive feedback frequency, while the unpopular items are for the rest. We can find AutoDebias provides more exposure opportunities to unpopular items, and the improvement mainly comes from unpopular items. This result validates that AutoDebias can generate fairer recommendation results.

\begin{figure}[t!]
    \centering
    \subfigure[]{
    \label{weight}
    \begin{minipage}[t]{0.33\linewidth}
    \centering
    \includegraphics[width=1in]{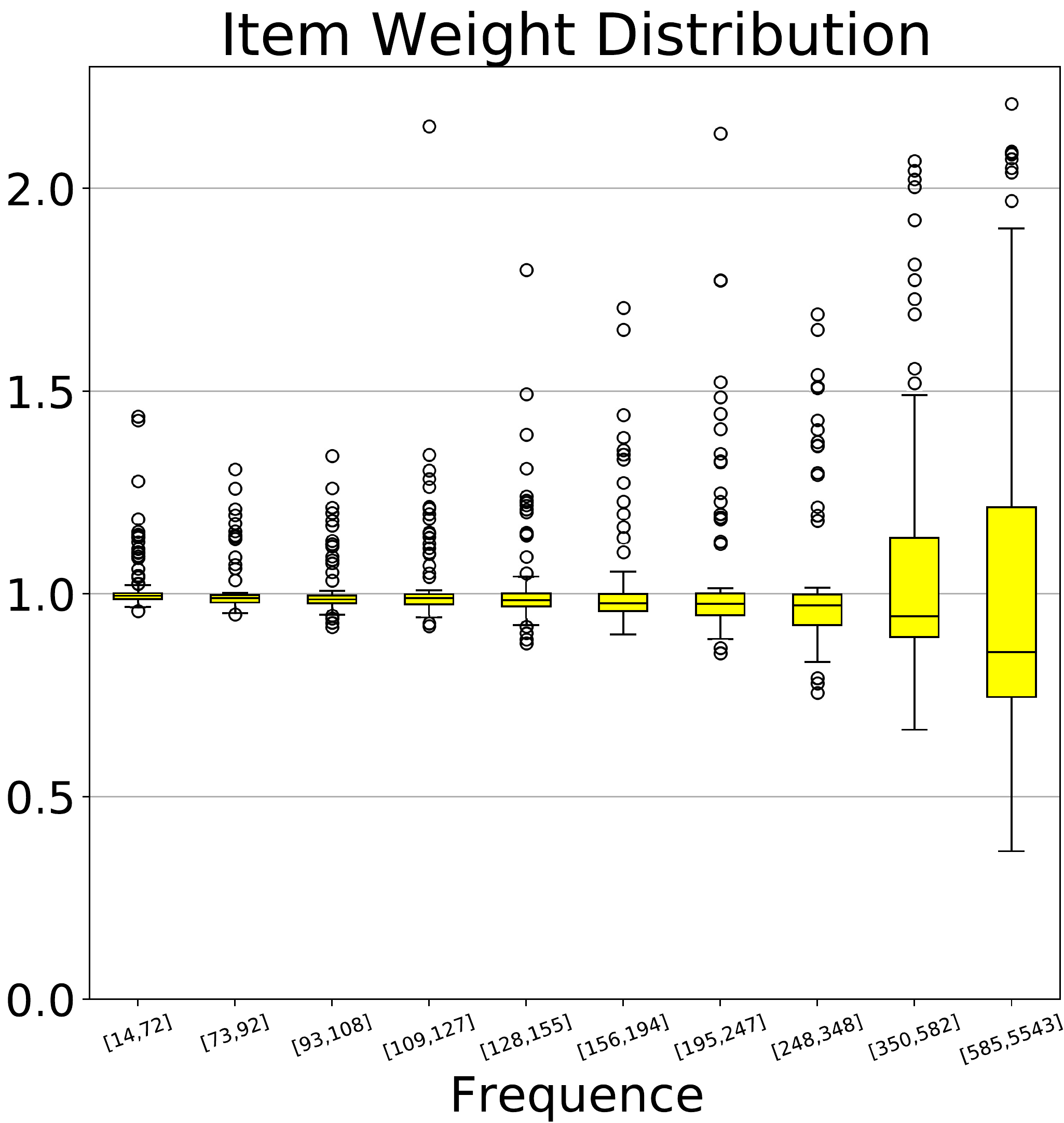}
    \end{minipage}
    }%
    \subfigure[]{
    \label{popular}
    \begin{minipage}[t]{0.33\linewidth}
    \centering
    \includegraphics[width=1in]{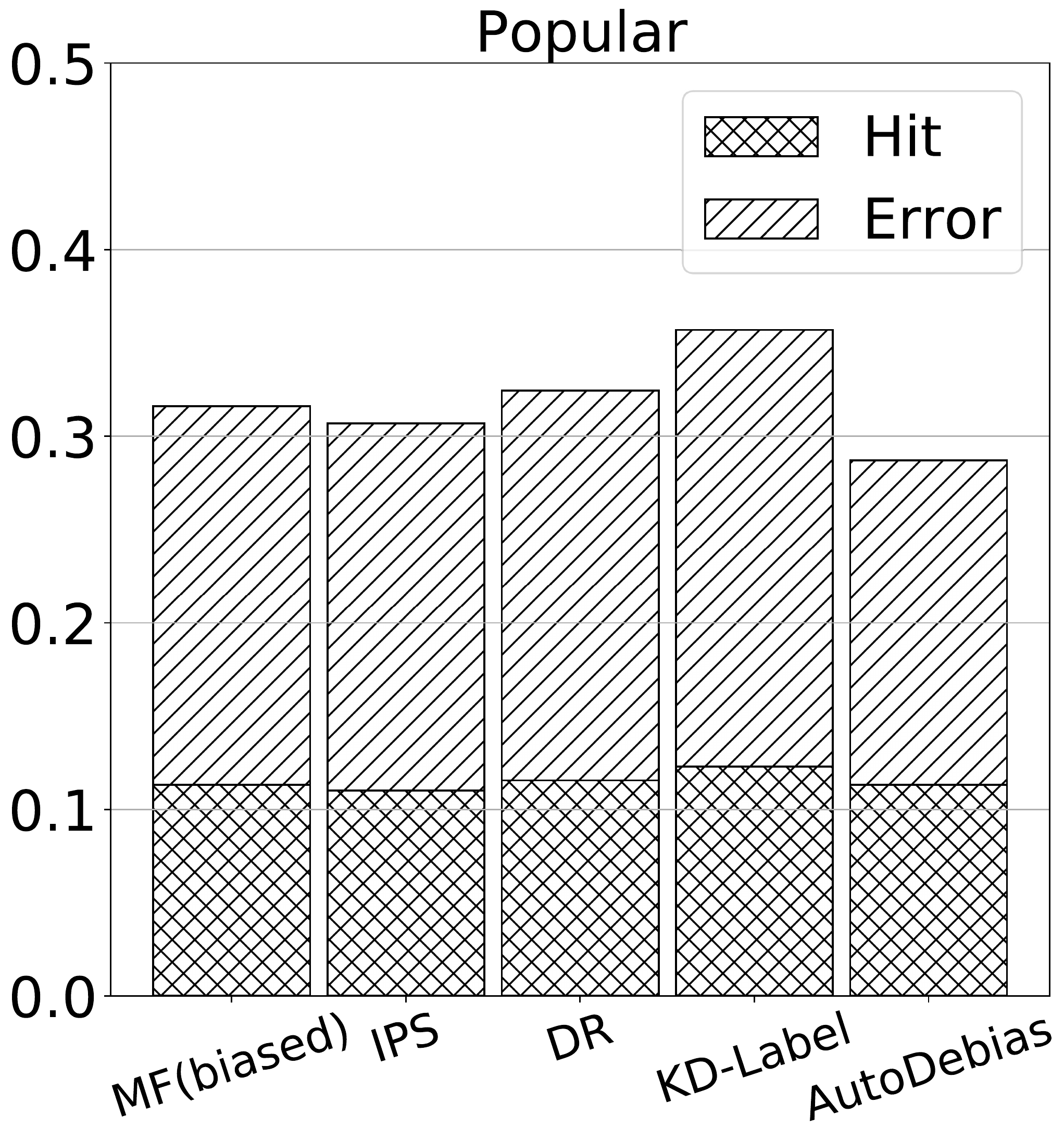}
    \end{minipage}%
    }%
    \subfigure[]{
    \label{unpopular}
    \begin{minipage}[t]{0.33\linewidth}
    \centering
    \includegraphics[width=1in]{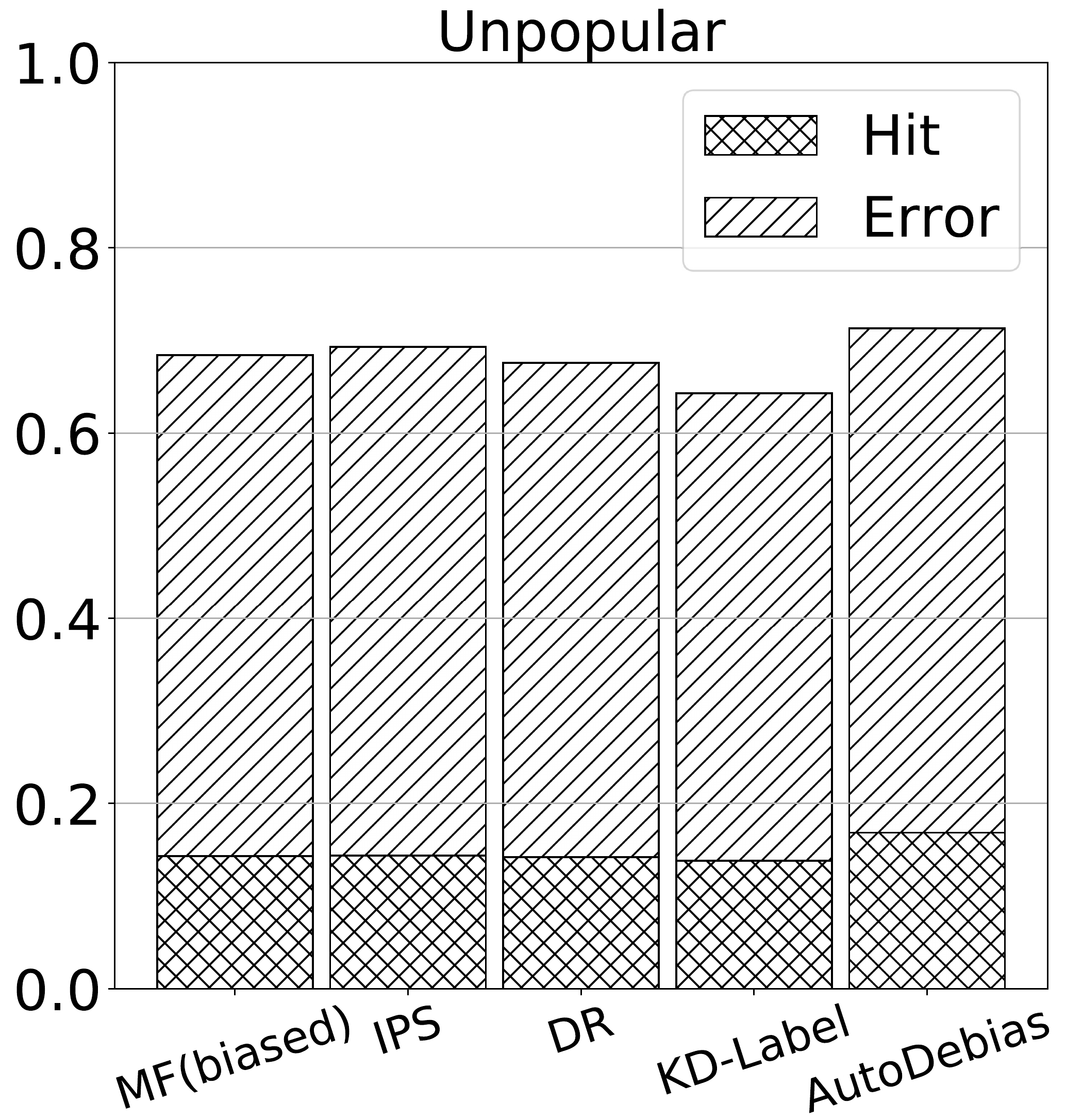}
    \end{minipage}%
    }%
    \centering
    \setlength{\abovecaptionskip}{-0.1cm}
    \caption{(a) The learned $w^{(1)}_i$ for each item $i$ with its popularity; (b) Recommendation results in terms of popular items; (c) Recommendation results in terms of unpopular items. }
     \vspace{-0.3cm}
\end{figure}

\subsection{Performance Comparison on Implicit and Feedback on Lists (RQ4)}
To validate the universality of AutoDebias, we also conduct experiments on implicit feedback and the feedback on lists.


\textbf{Implicit feedback.} We compare AutoDebias with three SOTA methods on exposure bias: WMF, RL-MF~\cite{saito2020unbiased}, and AWMF~\cite{chen2020fast}.
Table~\ref{result of implicit} shows the performance comparison.
As this table shows, AutoDebias performs much better than others, which validates our AutoDebias can handle the exposure bias in implicit feedback.

\textbf{The feedback on recommendation list.} In this scenario,
we add position information in modeling $w_k^{(1)}$, \ie
$w_k^{(1)} = \exp (\varphi_1^T[{\mathbf x_{{u_k}}}
\circ {\mathbf x_{i_k}} \circ {\mathbf e_{r_k} }
\circ {\mathbf e_{p_k}} ])$.
We compare AutoDebias with two SOTA methods: Dual Learning Algorithm (DLA)~\cite{DBLP:conf/sigir/AiBLGC18} which mitigate position bias,
and Heckman Ensemble (HeckE)~\cite{DBLP:conf/www/OvaisiAZVZ20} which mitigates both position bias and selection bias by ensemble.
Table~\ref{result of list} shows the results.
we can find AutoDebias still outperform compared methods. this result verifies that AutoDebias can handle various biases, even their combination.

\begin{table}[t]
    \setlength{\abovecaptionskip}{0.2cm}
 \setlength{\belowcaptionskip}{-0.00cm}
    \caption{The performance on implicit feedback data.}
    \label{result of implicit}
    \begin{tabular}{c|cc|cc}
    \hline
    \multirow{2}{*}{Method} & \multicolumn{2}{c}{Yahoo!R3-Im}  & \multicolumn{2}{c}{Coat-Im} \\ \cline{2-5}
                & AUC    & NDCG@5       & AUC     & NDCG@5   \\ \hline
    WMF              &  0.635       &     0.547     &    \textbf{0.749}      &      0.521        \\
    RI-MF         &   0.673   &   0.554     &    0.696   &    \textbf{0.527}     \\
    AWMF            &  0.675      &   0.578    &  0.614   &   0.505       \\
    AutoDebias      &   \textbf{0.730}    &   \textbf{0.635}    &   0.746    &   \textbf{0.527}          \\ \hline
\end{tabular}
  \vspace{-0.3cm}
\end{table}




\begin{table}[t]
   \setlength{\abovecaptionskip}{0.2cm}
 \setlength{\belowcaptionskip}{-0.00cm}
    \caption{The performance in the list feedback of compared methods on Simulation dataset.}
    \label{result of list}
    \begin{tabular}{cccc}
    \hline
            & NLL  & AUC    & NDCG@5      \\ \hline
    MF(biased)             &   -0.712   &   0.564    &      0.589     \\
    DLA                     &  -0.698     &   0.567   &     0.593    \\
    HeckE                &   -0.688    &    0.587  &     0.648   \\
    AutoDebias                 &   \textbf{-0.667}    &   \textbf{ 0.634}  &     \textbf{0.707}     \\ \hline
\end{tabular}
\vspace{-0.3cm}
\end{table}

\section{Related work}

\textbf{Bias in recommendation.} Besides the data bias that has been detailed in section 2\&3, two important biases in recommendation results have been studied: (1) When a recommender model is trained on a long-tailed data, popular items are recommended even more frequently than their popularity would warrant \cite{abdollahpouri2020multi}, raising so-called \textit{popularity bias}. The long-tail phenomenon is common in RS data, and ignoring the popularity bias will incur many issues, \eg, hurting user serendipity, making popular items even more popular. To deal with this problem, \cite{DBLP:conf/recsys/AbdollahpouriBM17,DBLP:conf/flairs/WasilewskiH16,DBLP:conf/sigir/ChenXLYSD20} introduced regularizers to guide the model to give more balanced results; \cite{zheng2020disentangling,wei2020model} disentangled the effect of user interest and item popularity with causal inference to address popularity bias. (2) The system systematically and unfairly discriminates against certain individuals or groups of individuals in favor others, raising so-called \textit{unfairness}. Unfairness happens as different user (or item) groups are usually unequally represented in data. When training on such unbalanced data, the models are highly likely to learn these over-represented groups, and potentially discriminates against other under-represented groups \cite{DBLP:conf/www/StoicaRC18,DBLP:conf/recsys/EkstrandTKMK18}. There are four types of strategies on addressing unfairness, including re-balancing\cite{DBLP:conf/kdd/SinghJ18}, regularizer\cite{DBLP:conf/nips/YaoH17}, adversarial learning~\cite{DBLP:conf/wsdm/BeigiMGAN020}, causal inference~\cite{Counterfactual-Fairness}. We encourage the readers refer to the survey~\cite{DBLP:journals/corr/abs-2010-03240} for more details.

\textbf{Meta learning in recommendation.} Meta learning is an automatic learning algorithms that aims at using metadata to improve the performance of existing learning algorithms or to learn (induce) the learning algorithm itself~\cite{DBLP:books/sp/1998TP, DBLP:conf/icann/HochreiterYC01, DBLP:conf/icml/FinnAL17}. There is also some work introducing meta learning in RS. For example, ~\cite{DBLP:conf/kdd/ChenC0GLLW19, DBLP:conf/wsdm/Rendle12} leveraged meta learning to learn a finer-grained regularization parameters; ~\cite{zhang2020retrain} leveraged meta-learning to guide the re-training of a recommendation model towards better performance; Meta-learning also has been utilized to address the cold-start problem \cite{lu2020meta}.

\section{Conclusion and Future work}

This paper develops a universal debiasing framework that not only can well account for multiple biases and their combinations, but also frees
human efforts to identify biases and tune the configurations.
We first formulate various data biases as a case of risk discrepancy, and then derive a general learning
framework that subsumes most debiasing strategies. We further propose a meta-learning-based algorithm to adaptively learn the optimal debiasing configurations from uniform data. Both theoretical and empirical analyses have been conducted to validate the effectiveness of our proposal.

One interesting direction for future work is to explore more sophisticate meta model, which could capture more complex patterns and potentially achieve better performance than linear meta model. Also, note that in real-world biases are usually dynamic rather than static. it will be valuable to explore how bias evolves with the time goes by and develop a universal solutions for dynamic biases.

\begin{acks}
This work is supported by the National Natural Science Foundation of China (U19A2079, 61972372),  National Key Research and Development Program of China (2020AAA0106000), USTC Research Funds of the Double First-Class Initiative (WK2100000019), and the Alibaba Group through Alibaba Innovative Research Program. 
\end{acks}

\bibliographystyle{ACM-Reference-Format}
\balance 
\bibliography{sigproc}

\appendix
\end{document}